\documentclass[sigconf]{acmart}
\usepackage{CJKutf8}
\usepackage{url}
\usepackage{booktabs}
\usepackage{enumitem}
\usepackage{makecell}
\usepackage{hyperref}
\usepackage{xcolor}
\usepackage{graphicx}
\usepackage{xspace}
\usepackage{multirow}

\newcommand{\modelname}{GBK-GNN}

\newcommand{\Fig}{Fig.\xspace}

\newcommand{\Eq}{Eq.\xspace}

\newtheorem{defn}{Definition}

\usepackage{bbm}
\AtBeginDocument{%
  \providecommand\BibTeX{{%
    \normalfont B\kern-0.5em{\scshape i\kern-0.25em b}\kern-0.8em\TeX}}}

\setcopyright{acmcopyright}
\copyrightyear{2022} 
\acmYear{2022} 
\setcopyright{acmcopyright}\acmConference[WWW '22]{Proceedings of the ACM Web Conference 2022}{April 25--29, 2022}{Virtual Event, Lyon, France}
\acmBooktitle{Proceedings of the ACM Web Conference 2022 (WWW '22), April 25--29, 2022, Virtual Event, Lyon, France}
\acmPrice{15.00}
\acmDOI{10.1145/3485447.3512201}
\acmISBN{978-1-4503-9096-5/22/04}

\begin{document}

\title{\modelname{}: Gated Bi-Kernel Graph Neural Networks for Modeling Both Homophily and Heterophily}

\author{Lun Du}
\authornote{Equal Contribution}
\authornote{Corresponding Author}
\email{lun.du@microsoft.com}
\affiliation{
    \institution{Microsoft Research Asia}
    \city{Beijing}
    \country{China}
}

\author{Xiaozhou Shi}
\authornotemark[1]
\authornote{Work performed during the internship at MSRA}
\email{xzh0u.sxz@gmail.com}
\affiliation{%
\institution{Beijing University of Technology}
\city{Beijing}
\country{China}
}

\author{Qiang Fu}
\authornotemark[1]
\email{qifu@microsoft.com}
\affiliation{
    \institution{Microsoft Research Asia}
    \city{Beijing}
    \country{China}
}

\author{Xiaojun Ma}
\authornotemark[3]
\email{mxj@pku.edu.cn}
\affiliation{
\institution{Peking University}
\city{Beijing}
\country{China}
}

\author{Hengyu Liu}
\authornotemark[3]
\email{hengyuliu94@gmail.com}
\affiliation{
    \institution{Northeastern University}
    \city{Shenyang}
    \country{China}
}

\author{Shi Han}
\author{Dongmei Zhang}
\email{{shihan, dongmeiz}@microsoft.com}
\affiliation{
\institution{Microsoft Research Asia}
\city{Beijing}
\country{China}
}

\begin{abstract}
  Graph Neural Networks (GNNs) are widely used on a variety of graph-based machine learning tasks. For node-level tasks, GNNs have strong power to model the homophily property of graphs (i.e., connected nodes are more similar), while their ability to capture heterophily property is often doubtful. This is partially caused by the design of the feature transformation with the same kernel for the nodes in the same hop and the followed aggregation operator. One kernel cannot model the similarity and the dissimilarity (i.e., the positive and negative correlation) between node features simultaneously even though we use attention mechanisms like Graph Attention Network (GAT), since the weight calculated by attention is always a positive value. In this paper, we propose a novel GNN model based on a bi-kernel feature transformation and a selection gate. Two kernels capture homophily and heterophily information respectively, and the gate is introduced to select which kernel we should use for the given node pairs. We conduct extensive experiments on various datasets with different homophily-heterophily properties. The experimental results show consistent and significant improvements against state-of-the-art GNN methods.
\end{abstract}




\maketitle
\begin{figure}[htbp]
\centering
\begin{minipage}[t]{0.20\textwidth}
\centering
\includegraphics[width=1.1\textwidth]{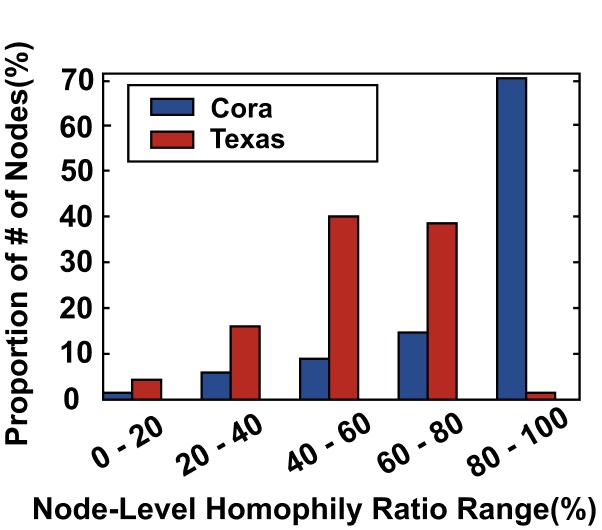}

\end{minipage}
\hspace{0.15in}
\begin{minipage}[t]{0.20\textwidth}
\centering
\includegraphics[width=1.1\textwidth]{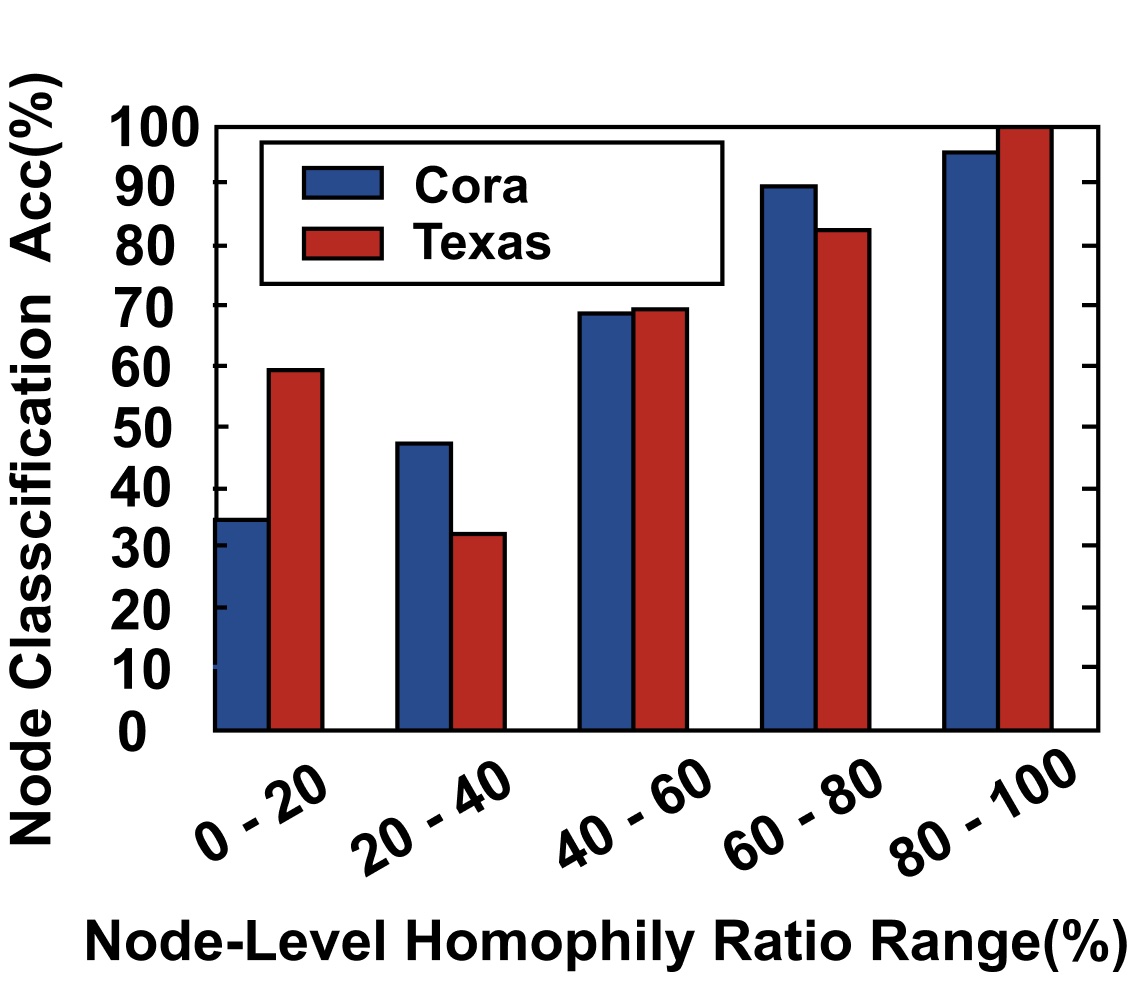}
\label{fig:range_homo}
\end{minipage}
\caption{The left figure illustrates the proportion of number of nodes ordered by node-level homophily ratio range. The right one illustrates node classification accuracy by GCN ordered by node-level homophily ratio range.}
\label{fig:observation}
\end{figure}
\section{Introduction}
Graph structured data are ubiquitous due to their vivid and precise depiction of relational objects. As a powerful approach for learning graph representations, Graph Neural Networks (GNNs), such as ChebNet \cite{defferrard2016convolutional}, GCN \cite{kipf2016semi}, GraphSAGE \cite{DBLP:journals/corr/HamiltonYL17} and GAT \cite{velivckovic2017graph}, are widely used on various applications ranging from social networks to biological networks \cite{xu2021self,du2021tabularnet,wei2021designing,wang2020cocogum}. The process of  node representation learning can be roughly summarized as two steps: (1) feature transformation via one learnable kernel; (2) aggregation of transformed feature from neighbors. The learned node representations can be fed to the downstream neural networks for specific learning tasks such as node classification, link prediction and community detection \cite{tang2011learning,bhagat2011node,long2019hierarchical,wang2019tag2gauss,wang2019tag2vec}. 

Although GNNs have achieved remarkable results on a wide range of scenarios,
their ability to model different properties of graphs has not been analyzed. Many previous studies show that GNNs are more effective in dealing 
with graphs with good \textbf{homophily} property (i.e., connected nodes are more similar), while their ability to capture \textbf{heterophily} property (i.e., connected nodes are more opposite) is often doubtful \cite{pei2019geom,zhu2020graph,zhu2020beyond,maurya2021improving,ma2021homophily,wang2020gcn}. 
In addition, even for the graphs with good homophily property, existing works do not model them well enough because they usually treat the homophily as a global property that could be consistently satisfied everywhere in the whole graph. However, they ignore the factor that there may be different levels of homophily among different local regions in a graph. Note that the homophily (and heterophily) property could usually be quantitatively measured with an index Homophily Ratio (HR) \cite{zhu2020beyond} that is formally defined as an average of the consistency of labels between each connected node pair. According to our analysis, the Homophily Ratio of different parts of the graph may have a big variance, which indicates the different homophily levels (corresponding to neighborhood patterns related to homophily-heterophily properties) among different sub-graphs. 

In more detail, we order the nodes on a graph by Node-level Homophily Ratio (NHR), i.e., Homophily Ratio within a subgraph consisting of a given node and the edges connected the node, to analyze the characteristics of local sub-graphs with the different homophily-heterophily properties. Cora \cite{Planetoid} and Texas \cite{pei2019geom} are often used as representatives of homophily graphs and heterophily graphs in literature, respectively. As shown on the left of Fig. \ref{fig:observation}, in Texas, the homophily level varies greatly around different nodes, and many nodes are in a mixed state (e.g., NHR in range 20\% - 80\%) where we cannot easily determine whether a subgraph of a given node is homophily or heterophily. Even in the homophily graph Cora, there are about 30\% nodes in the mixed state. On the right of Fig. \ref{fig:observation}, we illustrate the accuracy of the node classification with the GCN model for nodes with different NHR. GCN obtains impressive performance for high NHR nodes while suffering a rapid decline in effectiveness for nodes in the mixed state on both graphs. The observation hints that we should adaptively model the nodes with different NHR and improve the effectiveness for the nodes in a mixed state.

In this paper, we theoretically analyze the reason for GCN's failure for the graph whose nodes have different homophily levels, and propose a novel GNN model, namely Gated Bi-Kernel Graph Neural Network (\modelname{}), according to our observations and the identified problem of GCN. When using a GCN model, the representations of nodes in a mixed state will be difficult to distinguish, because only the same kernel is used to transform the features of neighbors with different labels, and the followed aggregation operation will smooth the distinguishable feature of different labels. To avoid smoothing distinguishable features, we propose to use bi-kernels, one for homophily node pairs and the other for heterophily node pairs. 

In addition, we introduce a selection gate to determine whether a given node pair is homophily or heterophily, and then select the corresponding kernel to use. This selection gate is learnable with the training objective as minimizing the prediction cross entropy loss between the node pair prediction (i.e., homophily pair or heterophily pair) with the ground-truth, which could be obtained according to the node labels in the training data. Overall, we have two kernels and a selection gate to learn. They are co-trained together to optimize the total loss, which is the weighted summation of a task loss and a gate prediction loss.
Furthermore, we theoretically prove the distinguishability of \modelname{} is better than GCN.

In summary, the main contributions of our work are as follows:
\begin{itemize}
    \item We first propose to analyze the homophily-heterophily properties in a fine-grained manner, i.e., from a local perspective instead of a global perspective, to expose the uneven homophily levels among different local sub-graphs. Moreover, we theoretically analyze the reason for GCN's low performance on nodes in the local regions with low homophily levels.
    \item We propose a model  called \modelname{} to tackle the challenges by introducing the selection gate and the two kernels to model homophily and heterophily, respectively. It can theoretically improve the discriminative ability compared with GCN without introducing many parameters and hyper-parameters.
    \item Extensive experiments are conducted on real-world datasets. 
    Our \modelname{} outperforms most of the baselines on \textbf{seven} graphs significantly with various Homophily Ratios, and the largest error rate deduction is 32.89\%. 
    Moreover, the evaluation results verify the model can balance the performance on nodes with various homophily-heterophily properties. 
\end{itemize}

\section{Preliminaries and Definitions}
Let $\mathbf{G} = (\mathbf{V}, \mathbf{E}, \mathbf{X}, \mathbf{Y})$ denote a graph with node features and labels, where $\mathbf{V} = \{v_i|i=1,2,\ldots, n\}$ is the node set with $n$ nodes and $\mathbf{E} = \{(v_i, v_j) |\; v_i, v_j \in \mathbf{V}\; {\rm and} \; v_i, v_j \; {\rm is\; connected} \}  $ is the edge set. $\mathbf{X} \in \mathbb{R}^{n \times \iota}$ is the node feature matrix where $\iota$ is the feature dimension, and $\mathbf{x}_i$ is the feature of node $v_i$. $\mathbf{Y}$ is the node label set and the label $y_i$ of $v_i$ satisfies that $\forall \ v_i \in \mathbf{V}, y_i \in \mathbb{Z}^{+}_0$ and $y_i < k$
, and $k$ is the number of categories. $\mathbf{A} \in \{0, 1\}^{n \times n}$ denotes the adjacency matrix of the graph $\mathbf{G}$ representing the graph connectivity information, and it satisfies if $(v_i, v_j) \in \mathbf{E}$ then $\mathbf{A}[i, j] = 1$  and otherwise $\mathbf{A}[i, j] = 0$. The neighbor set of a center node $v_i$ is denoted as $\mathcal{N}(v_i) = \{v_j|\;(v_i, v_j) \in \mathbf{E}\}$.

In this paper, we mainly focus on the Node Classification task, which is one of the most critical machine learning tasks on a graph and is highly related to the homophily-heterophily properties of the graph. Node classification is a semi-supervised learning problem, and it is formally defined as follows:
\begin{defn}[\textbf{\emph{Node Classification}}]
\label{def:node_classification}
	Node classification is a task to learn a conditional probability $\mathbf{\mathcal{Q}}(\mathbf{Y}|\mathbf{G}; \Theta)$ to distinguish the category of each unlabeled node on a single graph $\mathbf{G} = (\mathbf{V}, \mathbf{E}, \mathbf{X}, \hat{\mathbf{ Y}})$, where $\Theta$ is the model parameters and $\hat{\mathbf{Y}}$ is the partial observed node label set.
\end{defn}

\subsection{Graph Neural Networks}
Graph Neural Networks (GNNs) can be viewed as an effective node encoder by leveraging the neighbor information. The process of node representation learning can be summarized as two steps: (1) feature transformation via a learnable kernel; (2) aggregation of transformed feature from neighbors. A general GNN architecture can be formulated as:
\begin{equation}
    \mathbf{z}_i^{\left (l \right)} = \gamma^{(l)} \left ( \mathbf{z}_i^{(l-1)}, \Gamma_{v_j\in\mathcal{N}(v_i)}\phi^{(l)}\left ( \mathbf{z}_i^{(l-1)}, \mathbf{z}_j^{(l-1)} \right)  \right ) 
\end{equation}
where $\Gamma$ denotes a differentiable, permutation invariant aggregation operation, e.g. mean, sum, or max. $\gamma$ and $\phi$ denote differentiable functions like Multilayer Perceptron (MLP). 
$\mathbf{z}_i^{(l)}$ is the output of node $v_i$ at the $l$-th layer.
After the GNN encoder, a linear fully connected layer is built above for classification. 

Graph Convolutional Network (GCN) \cite{kipf2016semi} is one of the most popular GNN models. In detail, at the $l$-th layer, GCN use the following formulation to transform the input:
\begin{equation}
\begin{split}
    \mathbf{Z}^{\left (l \right)} = \sigma(\hat{\mathbf{A}}\mathbf{Z}^{\left (l-1 \right)}\mathbf{W}^{l})\\
    {\rm where}\; \hat{\mathbf{A}} = \mathbf{D}^{-1}(\mathbf{A}+\mathbf{I}).
\end{split}
\end{equation}
$\mathbf{D}$ is a diagonal matrix and $D[i, i] = |\mathcal{N}(v_i)| + 1$ is the degree of the node $v_i$ plus one, and $\mathbf{I}$ is an identity matrix.

With the reasonable assumption of the conditional in-dependency, the target conditional distribution is actually equivalent to being decomposed, i.e., $\mathbf{\mathcal{Q}}(\mathbf{Y}|\mathbf{G}; \Theta) = \prod_{v_i \in \mathbf{V}} q(y_i|x_i, \mathcal{N}^{(l)}(v_i); \Theta)$ and $\mathcal{N}^{(l)}(v_i)$ is the neighbor set of the node $v_i$ within $l$-hops. In other words, only the node feature and the local neighbor information are utilized for classifying a certain node when using GCN to solve the node classification problem.

\subsection{Homophily vs. Heterophily in Graphs}
Homophily is one of the natural properties of graphs. A graph with good homophily property means that the connected node pairs tend to be more similar. To be specific for the node classification problem, it means the node pairs share the same label with high probability. In contrast, the connected node pairs are more dissimilar on a graph under heterophily. To quantitatively measure the homophily-heterophily properties, Homophily Ratio (HR) \cite{zhu2020beyond} is defined as follows:
\begin{defn}[\textbf{\emph{Homophily Ratio}}]
\label{def:homophily_ratio}
	The homophily ratio of a given graph $\ \mathbf{G} = (\mathbf{V}, \mathbf{E}, \mathbf{X}, \mathbf{Y})$ is defined as
	\begin{equation}
	    H(\mathbf{G}) = \frac{\sum_{(v_i, v_j) \in \mathbf{E}} \mathbbm{1}(y_i = y_j)}{|\mathbf{E}|},
	\end{equation}
	where $\mathbbm{1}(\cdot)$ is the indicator function (i.e., $\mathbbm{1}(\cdot) = 1$ if the condition $\cdot$ holds, otherwise $\mathbbm{1}(\cdot) = 0$), and $|\mathbf{E}|$ is the size of the edge set $\mathbf{E}$.
\end{defn}

HR can measure the global homophily property of the whole graph. In literature, when $H(\mathbf{G}) \to 1$ the graph $\mathbf{G}$ is considered to be strong homophily, while the graph $\mathbf{G}$ is considered to be strong heterophily when $H(\mathbf{G}) \to 0$. 

According to our analysis, the Homophily Ratio of different parts of the graph may have a big variance, which indicates the different homophily levels (corresponding to neighborhood patterns related to homophily-heterophily properties) among different sub-graphs. In addition, when using the GNN model to solve a node classification problem, only local neighbor information will be considered for classifying the center node. In order to study the homophily property from a fine-grained local perspective, it is reasonable to define a Node-level Homophily Ratio: 
\begin{defn}[\textbf{\emph{Node-level Homophily Ratio}}]
\label{def:node_homophily_ratio}
	Given a center node $v_i$ in a given graph $\mathbf{G} = (\mathbf{V}, \mathbf{E}, \mathbf{X}, \mathbf{Y})$, the
	node-level homophily ratio is defined as the fractions of edges that connect the node $v_i$ and another node $v_j$ with the same label. It is formally defined as:
	\begin{equation}
	    h_i = \frac{\sum_{v_j \in \mathcal{N}(v_i)} \mathbbm{1}(y_i = y_j)}{|\mathcal{N}(v_i)|},
	\end{equation}
	where $|\mathcal{N}(v_i)|$ is the size of the neighbor set $\mathcal{N}(v_i)$.
\end{defn}

\section{Generalization Bound of GCN with Regard to Homophily Ratio }
In this section, we theoretically analyze the generalization ability of Graph Convolutional Networks (GCNs) with regard to different homophily-heterophily levels and further theoretically explain the reason for GCN's failure in classifying the nodes with mixed homophily and heterophily neighbors. 

\begin{figure*}[!ht]
\includegraphics[width=1.0\textwidth]{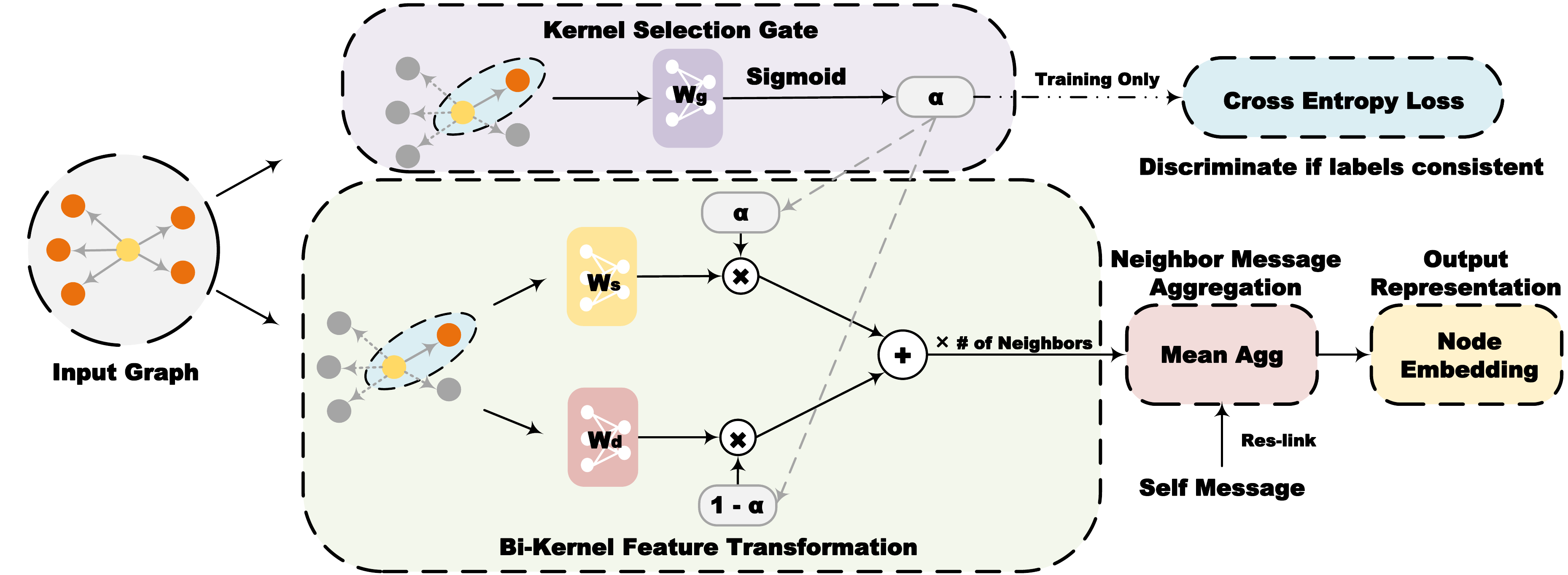}
\caption{An overview of \modelname{}. \modelname{} has two main modules: kernel selection gate learning and bi-kernel feature transformation. The learning process of the kernel selection gate receives input graph information and outputs selective signal $\alpha$ to discriminate if the neighbor node labels are consistent. The bi-kernel feature transformation trains \textbf{$W_s$} and \textbf{$W_d$}, namely weights to capture the similarity between nodes and weights to capture the dissimilarity between nodes. It uses the signal $\alpha$ from the former module to combine these two \textbf{$W$} in the process of message passing, then doing mean aggregation, finally producing node embedding. In the training phase, we have an additional cross-entropy loss to train the selection gate with supervisions.} 
\label{fig:overview}
\end{figure*}

\subsection{Complexity Measure Based View for Neural Networks Generalization}
In order to theoretically analyze the generalization ability of GCN so as to further explain the reason of GCN's failure for classifying nodes with mixed homophily and heterophily neighbors, the first step is to find a metric to measure the generalization ability of GCN with a given Homophily Ratio. Complexity Measure is the current mainstream method to measure the generalization ability of the model \cite{neyshabur2017exploring}. As we know, a model that can fit real data with lower capacity often has better generalization ability according to Occam's Razor principle. Complexity Measure is defined to measure the capacity of neural networks, so a \textbf{lower complexity measure} means a \textbf{better generalization ability}. Formally, a Complexity Measure is a measure function $\mathcal{M}:\{\mathcal{H}, \mathcal{S}\}\to \mathbb{R}^+$ where $\mathcal{H}$ is a class of models and $\mathcal{S}$ is a training set. In the following theoretical analysis, we can set $\mathcal{H}$ as GCNs with different parameters and $\mathcal{S}$ as graph data with a certain Homophily Ratio.

There are different forms of Complexity Measure $\mathcal{M}$ specifically designed from different aspects, such as sharpness of a local minimum \cite{dziugaite2017computing,keskar2016large} and the scale of norms of parameters \cite{neyshabur2015norm,neyshabur2015path}. We select Consistency of Representations \cite{natekar2020representation}\footnote{This work won the Winning Solution of the NeurIPS 2020 Competition on Predicting Generalization in Deep Learning.} as our Complexity Measure. This measure is designed based on Davies-Bouldin Index \cite{davies1979cluster}. Mathematically, for a given dataset and a given layer of a model,
\begin{align}
     S_i = \Big( \frac{1}{n_i} \sum^{n_i}_\tau | \mathcal{O}_i^{(\tau)} - \mu_{\mathcal{O}_i} |^p \Big)^{1/p} \; &{\rm for}\; i = 1 \cdots k \label{equ:intra-variance}\\
     M_{i, j} = || \mu_{\mathcal{O}_i} - \mu_{\mathcal{O}_j} ||_p \qquad &{\rm for}\; i,j = 1 \cdots k \label{equ:inter-variance},
\end{align}
where $i$ and $j$ are two different classes, $\mathcal{O}_i^{(\tau)}$ is the output representation of the $\tau$-th sample belonging to class $i$ for the given layer, $\mu_{\mathcal{O}_i}$ is the cluster centroid of the representations of class $i$, $S_i$ is a measure of scatter within representations of class $i$, and $M_{i, j}$ is a measure of separation between representations of classes $i$ and $j$. Then, the complexity measure based on the Davies Bouldin Index:
\begin{equation}
\label{equ:complexity_measure}
    \mathcal{C} = \frac{1}{k} \sum_{i = 0}^{k-1} \max_{i \neq j} \frac{S_i + S_j}{M_{i, j}}.
\end{equation}
When $p$ is set as 2. This complexity measure can be roughly interpreted as the ratio of the intra-class variance to the inter-class variance.

\subsection{Generalization Bound Analysis and Failure Reason of GCN}
We use the complexity measure of Davies-Bouldin Index with $p = 2$ (refer to \Eq{}\eqref{equ:intra-variance} to \eqref{equ:complexity_measure}) to analyze the generalization ability of GCNs. We propose several assumptions to simplify the analysis: (1) a binary classification problem is considered (i.e., $k=2$); (2) given the category $y_i$ of the center node $v_i$, the probability of its neighbors belonging to the same category is $P_{y_i}$, e.g., $P_0$ means for a center node belonging to the 0-th class, the probability of its neighbors belonging to the same category; (3) The number of neighbors is $d$ for all the nodes in the graph; (4) The graph has been added self-loop to facilitate GCN's calculation; (5) We ignore the non-linear activation in GCNs for brevity.
Then we have the following theorem:
\begin{theorem}
\label{theorem:1}
For a binary node classification problem and the input graph satisfies the assumptions above, if $|P_0 + P_1 - 1| \to 0$ then the complexity measure of Consistency of Representations will converge to \textbf{$+\infty$} for GCNs with an arbitrary kernel (except for a zero matrix). In that case, GCN will lose the generalization ability.
\end{theorem}
\begin{proof}
Since $k=2$ and $p=2$, \Eq{}\eqref{equ:complexity_measure} can be simplified as:
    $\mathcal{C} = \frac{S_0 + S_1}{M_{0, 1}}$.
We first calculate $\mu_{\mathcal{O}_i}$:
\begin{equation}
\begin{split}
    \mu_{\mathcal{O}_0} &= \mathbb{E}[\mathcal{O}_0^{(i)}] =  \mathbb{E}[\mathbf{W} \sum_{j \in \mathcal{N}(v_i)}\frac{1}{d}\mathbf{X}^{(j)}]\\
    &= \mathbf{W} (P_0 \cdot \mu_{\mathbf{X}_0} + (1 - P_0) \cdot \mu_{\mathbf{X}_1}),
\end{split}
\end{equation}
where $\mathbf{X}^{(j)}$ is the $j$-th node feature and $\mu_{\mathbf{X}_i}$ is the cluster centroid of the node features of class $i$.
Similarly, we have:
\begin{equation}
\mu_{\mathcal{O}_1} = \mathbf{W} (P_1 \cdot \mu_{\mathbf{X}_1} + (1 - P_1) \cdot \mu_{\mathbf{X}_0}).
\end{equation}
Then, $M_{0, 1}$ can be obtained by:
\begin{equation}
\label{equ:inter-class_variance}
\begin{split}
M_{0, 1} &= ||\mu_{\mathcal{O}_0} - \mu_{\mathcal{O}_1}||\\
&= ||\mathbf{W} (P_0 \cdot \mu_{\mathbf{X}_0} + (1 - P_0) \cdot \mu_{\mathbf{X}_1} - P_1 \cdot \mu_{\mathbf{X}_1} - (1 - P_1) \cdot \mu_{\mathbf{X}_0})||\\
&= (P_0 + P_1 - 1) \cdot ||\mathbf{W} (\mu_{\mathbf{X}_0} - \mu_{\mathbf{X}_1})||.
\end{split}
\end{equation}
Then $S_0^2$ is calculated by:
\begin{equation}
    \begin{split}
        S_0^2 &= \mathbb{E} [||\mathcal{O}_0^{(i)} - \mu_{\mathcal{O}_0}||^2] = \mathbb{E} [<\mathcal{O}_0^{(\tau)} - \mu_{\mathcal{O}_0},\; \mathcal{O}_0^{(\tau)} - \mu_{\mathcal{O}_0}>]\\
        &= P_0^2 \mathbb{E}\left[||\mathbf{W}(\mathbf{X}_0^{(i)} - \mu_{\mathbf{X}_0})||^2\right] + (1-P_0)^2\mathbb{E}\left[||\mathbf{W}(\mathbf{X}_1^{(i)} - \mu_{\mathbf{X}_1})||^2\right],
    \end{split}
\end{equation}
where $<\cdot,\cdot>$ is inner production. Let $\sigma_0^2 = \mathbb{E}\left[||\mathbf{W}(\mathbf{X}_0^{(i)} - \mu_{\mathbf{X}_0})||^2\right]$ and $\sigma_1^2 = \mathbb{E}\left[||\mathbf{W}(\mathbf{X}_1^{(i)} - \mu_{\mathbf{X}_1})||^2\right]$, then the above function can be rewritten as:
\begin{equation}
\begin{split}
    S_0^2 &= P_0^2 \sigma^2_0 + (1 - P_0)^2 \sigma^2_1
    \geq \frac{\sigma_0^2\sigma_1^2}{\sigma_0^2 + \sigma_1^2}.
\end{split}
\end{equation}
Similarly, we have:
\begin{equation}
\begin{split}
    S_1^2 = P_1^2 \sigma^2_1 + (1 - P_1)^2 \sigma^2_0
    \geq  \frac{\sigma^2_0\sigma_1^2}{\sigma_0^2 + \sigma_1^2}.
\end{split}
\end{equation}
Then we calculate the complexity measure $\mathcal{C}$:
\begin{equation}
\begin{split}
\mathcal{C} &= \frac{\sqrt{S_0^2 + S_1^2 + 2S_0 \cdot S_1}}{M_{0, 1}} \geq \frac{2\sigma_0\sigma_1}{\sqrt{\sigma_0^2 + \sigma_1^2} \cdot M_{0, 1}}\\
&= \frac{2\sigma_0\sigma_1}{|(P_0 + P_1 - 1)| \cdot \sqrt{\sigma_0^2 + \sigma_1^2} \cdot ||\mathbf{W} (\mu_{\mathbf{X}_0} - \mu_{\mathbf{X}_1})||} 
\end{split}
\end{equation}
Thus, we obtain a lower bound of $\mathcal{C}$. Notice that $\sigma_0$ and $\sigma_1$ could not be zero, otherwise the classification problem is meaningless. If $|P_0 + P_1 - 1| \to 0$, then the lower bound of $\mathcal{C} \to +\infty$, and hence $\mathcal{C} \to +\infty$. The model will lose the generalization ability. 
\end{proof}

\paragraph{Intuition of Theorem. \ref{theorem:1}} Theorem. \ref{theorem:1} tells us that if there are a similar number of homophily neighbors and heterophily neighbors for graph nodes, then GCNs will smooth the outputs from different classes and loss the discrimination ability even though the initial node features are quite distinguishable. Such inter-class smoothness is also a reason why GCNs cannot performance better than a naive Multilayer Perceptron (MLP) for the graph whose nodes have different homophily levels. If we look into the proof, we can find that the distance between representations of 0-th class and 1-th class $M_{i,j}$ tends to be 0, which makes the complexity measure $\mathcal{C}$ tend to be infinity. Thus, to avoid $M_{i,j}$ being zero inspires us for our proposed new model Gated Bi-Kernel Graph Neural Networks.

\section{Gated Bi-Kernel Graph Neural Networks}

According to the identified problem in GCNs, we propose a new model pertinently, namely Gated Bi-Kernel Graph Neural Networks (GBK-GNN). The overview of GBK-GNN is shown in \Fig{} \ref{fig:overview}. There are two main differences compared with vanilla GNNs, i.e., Bi-kernel Feature Transformation and Kernel Selection Gate modules.
\subsection{Bi-Kernel Design}
As mentioned above, the GCN's problem of lack of distinguishability is mainly caused by a small inter-class variance $M_{i, j}$. A single kernel of GCN cannot adaptively adjust the weights for different types of nodes according to different homophily properties. Mathematically, a single kernel GCN cannot change the proportion of $P_0$ and $P_1$ in \Eq{}\eqref{equ:inter-variance}. Thus, we design a Bi-Kernel feature transformation method to tackle the problem. 

In detail, there are two kernels in our model, one for homophily node pairs and the other for heterophily pairs. In an ideal case where we can exactly determine whether a pair is homophily or not, the inter-class variance $M_{0,1}$ will be changed to:
\begin{equation}
\label{equ:inter_variance_new}
    ||(P_1\mathbf{W}_s + (P_0-1)\mathbf{W}_d)(\mu_{\mathbf{X}_0} - \mu_{\mathbf{X}_1}) + (P_0 - P_1)(\mathbf{W}_s - \mathbf{W}_d)\mu_{\mathbf{X}_0}||,
\end{equation}
where $\mathbf{W}_s$ is the kernel for homophily edges and $\mathbf{W}_d$
is the kernel for heterophily edges. The first term in Eq. \eqref{equ:inter_variance_new} shows that $\mathbf{W}_s$ and $\mathbf{W}_d$ can adjust the relation of $P_1$ and $P_0-1$ to avoid appearance of $|P_1 + P_0 - 1|$ term. Besides, the second term can support extra distinguishability even if the original features lack discrimination, i.e., $||\mu_{\mathbf{X}_1} - \mu_{\mathbf{X}_0}|| \to 0$. 

Compared with Graph Attention Network (GAT), our design can better model the positive and negative correlations by setting one kernel as a positive definite matrix and the other kernel as a negative definite matrix. However, GAT cannot simultaneously model positive and negative correlations because the weights calculated by attention are always positive.

\subsection{Selection Gate Mechanism}
In reality, we cannot directly determine whether a node pair belongs to the same class or not. 
Thus, we introduce a learnable kernel selection gate to discriminate node pairs and adaptively select kernels. 
The formal form of the transformation of the input is listed as follows:
\begin{equation}
\begin{split}
    \mathbf{z}_i^{(l)} = \sigma \Big( \mathbf{W_f}\mathbf{z}_i^{(l-1)} &+ \frac{1}{N}\sum_{v_j \in \mathcal{N}(v_i)} \!\!\!\!\! \alpha_{ij} \mathbf{W_s}\mathbf{z}_j^{(l-1)} + (1 - \alpha_{ij})\mathbf{W_d}\mathbf{z}_j^{(l-1)} \Big) \\
    \alpha_{ij} &= {\rm Sigmoid}\left(\mathcal{G}_l
    (\mathbf{z}_i^{(l-1)}, \mathbf{z}_j^{(l-1)}; \; \mathbf{W_g})\right),
\end{split}
\end{equation}
where $\mathcal{G}_l(\cdot, \cdot)$ is a learnable function ($\mathbb{R}^{n_{l-1}} \times \mathbb{R}^{n_{l-1}} \to \mathbb{R}$) in the $l$-th layer, $\mathbf{W_s}$, $\mathbf{W_f}$, $\mathbf{W_d}$ and $\mathbf{W_g}$ are learnable parameters, $\mathbf{z}_i^{(l)}$ is the hidden representation of node $v_i$ in the $l$-th layer, $\alpha_{ij}$ is the gate signal, and $n_l$ is the dimension number of the node hidden representation in $l$-th layer. $\mathcal{G}_l(\cdot, \cdot)$ can be a multilayer perceptron or a graph neural network (e.g. GraphSage \cite{DBLP:journals/corr/HamiltonYL17}).
 
\subsection{Loss Design and Optimization}
Similar to vanilla classification methods, a cross-entropy loss $\mathcal{L}_o$ is used for the node classification problem. The difference is that we utilize an additional cross-entropy loss $\mathcal{L}_g^{(l)}$ for each layer $l$ to guide the training procedure of the selection gate. Thus, the final objective is a combination of two kinds of losses:
\begin{equation}
    \mathcal{L} = \mathcal{L}_o + \lambda \sum_l^L \mathcal{L}_g^{(l)}
\end{equation}
where $\lambda$ is a hyper-parameter to balance two losses. We use AdamW \cite{loshchilov2017decoupled} to optimize the loss $\mathcal{L}$.

\section{Experiments}

In this section, we evaluate the empirical performance of \modelname{} on real-world homophily and heterophily graph datasets on the node classification task and compare with other state-of-the-art models.
\subsection{Datasets}
We perform our experiments on seven datasets from PyTorch-Geometric, which are commonly used in graph neural network literature. An introduction of the datasets is presented in Table \ref{tab:datasets}. Homophily ratio \cite{zhu2020beyond} represents the fraction of edges that connect two nodes of the same label. If the ratio is close to 1, that means the dataset is homophily. And if the ratio is close to 0, that means the dataset is heterophily. 

\textbf{Cora, CiteSeer and PubMed} \cite{DBLP:journals/corr/YangCS16} are citation networks based datasets. In these datasets, nodes represent papers, and edges represent citations of one paper by another. Node features are the bag-of-words representation of papers, and node label denotes the academic topic of a paper. Since these citation datasets have a high homophily ratio, they are considered homophily datasets. 

\textbf{Wisconsin, Cornell and Texas} are subdatasets of WebKB collected by Carnegie Mellon University. In these datasets, nodes represent web pages and edges represent hyperlinks between them. Node features are the bag-of-words representation of web pages. The task is to classify the nodes into one of the five categories, i.e., student, project, course, staff, and faculty. 

\textbf{Actor} \cite{pei2019geom} is the actor-only induced subgraph of the film-director-actor-writer network. Each node corresponds to an actor, and the edge between two nodes denotes co-occurrence on the same Wikipedia page. Node features correspond to some keywords in the Wikipedia pages. The task is to classify the nodes into five categories in terms of words of actor’s Wikipedia.

These two kinds of datasets has a relatively low homophily ratio and are considered as heterophily datasets. 
\begin{table}[h]
\centering
\caption{Statistics of the node classification datasets.}
\begin{tabular}{lccccc}
\hline
Dataset & Hom. Ratio & Nodes  & Edges    & Features & Classes\\ \hline
Cora     &  0.819    & 2,708  & 10,556    & 1,433 & 7\\ 
CiteSeer &  0.703    & 3,327   & 9,104     & 3,703 &  6\\ 
PubMed   &  0.791    & 19,717  & 88,648    & 500 &   3\\ 
Cornell  &  0.207    & 183    & 298      & 1,703 &  5\\ 
Texas    &  0.114    & 183    & 325      & 1,703 &  5\\  
Wisconsin&  0.139    & 251    & 515      & 1,703 &  5\\  
Actor    &  0.201    & 7,600   & 30,019    & 932 &   5\\  \hline

\end{tabular}
\label{tab:datasets}
\end{table}

\begin{table*}[h]
\centering
\caption{Mean classification accuracy of our \modelname{} and other popular GNN models on both homophily and heterophily graph datasets. The best performing method is highlighted.}
\renewcommand\tabcolsep{2.0pt}
% \begin{adjustbox}{width=\textwidth}
\begin{tabular}{lccccccccc}
\hline
Model   & Cora                       & CiteSeer & PubMed   & Cornell & Texas & Wisconsin &  Actor\\ \hline
 DNN  &74.64 $\pm$ 0.63              &71.30 $\pm$ 1.94 &87.87 $\pm$ 0.32 & 67.74 $\pm$ 3.68  & 74.19 $\pm$ 5.22  &71.74 $\pm$ 6.93  &35.41 $\pm$ 0.93 \\
 GCN  & 87.22 $\pm$ 0.76 &74.67 $\pm$ 1.33 &87.40 $\pm$ 0.18 &35.48 $\pm$ 4.60  & 48.39 $\pm$ 4.15  &50.00 $\pm$ 5.58 &27.67 $\pm$ 1.33 \\
 GAT  &84.31 $\pm$ 1.21              &71.89 $\pm$ 1.40 &86.74 $\pm$ 0.08 &45.16 $\pm$ 5.56  & 41.94 $\pm$ 5.33  &63.04 $\pm$ 5.33 &26.56 $\pm$ 0.98 \\
 GIN  &81.93 $\pm$ 0.57              &68.08 $\pm$ 0.57 &85.76 $\pm$ 0.29 &38.71 $\pm$ 5.16  &29.03 $\pm$ 4.42  &45.65 $\pm$ 4.71  &23.61 $\pm$ 0.76 \\
GCNII  & 87.22 $\pm$ 0.98 &75.11 $\pm$ 1.18 &86.19 $\pm$ 0.57 &45.16 $\pm$ 4.87  & 51.61 $\pm$ 5.47  &58.70 $\pm$ 3.25  &28.72 $\pm$ 0.60 \\
GraphSage  & \underline{87.77 $\pm$ 1.83}         & \underline{76.72 $\pm$ 0.64} &87.06 $\pm$ 0.74 &\underline{70.97 $\pm$ 4.76} &70.97 $\pm$ 4.46  &73.91 $\pm$ 6.79 &34.10 $\pm$ 0.65 \\
\hline
CPGNN      &79.40 $\pm$ 1.39         & 69.41 $\pm$ 0.45 &77.40 $\pm$ 0.57 & 70.27 $\pm$ 5.12   & \underline{75.68 $\pm$ 5.12} & \underline{76.47 $\pm$ 6.16} &  \underline{35.59 $\pm$ 0.86} \\
H2GCN      &82.70 $\pm$ 0.87         & 43.70 $\pm$ 0.24  &78.60 $\pm$ 0.62  & 70.27 $\pm$ 4.99   & 72.97 $\pm$ 4.53 & 70.59 $\pm$ 5.76  &  35.39 $\pm$ 1.34 \\

GEOM-GCN   & 84.10 $\pm$ 1.12        &76.28 $\pm$ 2.06   &\underline{88.13 $\pm$ 0.67} & 54.05 $\pm$ 3.87 & 67.57 $\pm$ 5.35 & 68.63 $\pm$ 4.92 & 30.00 $\pm$ 1.26 \\  \hline

\modelname{}& \textbf{88.69 $\pm$ 0.42}&\textbf{79.18 $\pm$ 0.96}  &\textbf{89.11 $\pm$ 0.23} &\textbf{74.27 $\pm$ 2.18}  &\textbf{81.08 $\pm$ 4.88}  &\textbf{84.21 $\pm$ 4.33}    &\textbf{38.97 $\pm$ 0.97}  \\ \hline
Error Reduction    & \textbf{7.52\%}  &\textbf{10.57\%}  &\textbf{8.26\%}  &\textbf{11.37\%}  &\textbf{22.20\%}  &\textbf{32.89\%}     &\textbf{5.25\%}  \\ \hline
\end{tabular}
% \end{adjustbox}
\label{tab:baseline}
\end{table*}

\subsection{Baselines}
We compare our \modelname{} with three types of state-of-the-art methods for hetermophily graphs and six graph neural network based methods. 
\begin{itemize}
    \item \textbf{CPGNN} \cite{CPGNN} incorporates into GNNs a compatibility matrix that captures both heterophily and homophily by modeling the likelihood of connections between nodes in different classes.
    \item \textbf{H2GCN} \cite{zhu2020beyond} is a graph neural network combines ego and neighbor-embedding separation, higher-order neighborhoods, and combination of intermediate representations, which performs better on heterophily graph data.
    \item \textbf{GEOM-GCN} \cite{pei2019geom} is a graph convolutional network with a geometric aggregation scheme which operates in both graph and latent space.
    \item \textbf{DNN} is a basic fully connected neural network model.
    \item \textbf{GCN} \cite{kipf2016semi} is a semi-supervised graph convolutional network model which learns node representations by aggregating information from neighbors.
    \item \textbf{GAT} \cite{velivckovic2017graph} is a graph neural network model using attention mechanism to aggregate node features.
    \item \textbf{GIN} \cite{gin} is a graph neural network which can distinguish graph structures more as powerful as the Weisfeiler-Lehman test.
    \item \textbf{GCNII} \cite{DBLP:journals/corr/abs-2007-02133} is an extension of graph convolutional network with initial residual and identity mapping which can relieves the problem of over-smoothing.
    \item \textbf{GraphSage} \cite{DBLP:journals/corr/HamiltonYL17} is a general inductive framework that leverages node feature information to efficiently generate node embeddings for previously unseen data
    
\end{itemize}

\subsection{Experimental Setup}
We use a public data split provided by \cite{pei2019geom}, the nodes of each class are 60\%, 20\%, and 20\% for training, validation and testing. 
 
We use Adam optimizer and fix the hidden unit to be 16, running 500 epochs for each experiment. For other hyper-parameters, we perform a hyper-parameter search for all models. For fairness, the size of the search space for each model is the same. The searching hyper-parameters include learning rate, weight decay, and $\lambda$ in our regularization term. The hyper-parameter search process ensures each model has the same search space. Learning rate is searched in $\{1e-3, 1e-4, 1e-5\}$, weight decay is searched in $\{1e-2, 1e-3, 1e-4\}$, and $\lambda$ is searched in $(0, 64]$. The experimental setup for \modelname{} and all other baselines are the same except for the layer number of GCNII. We use 64 layers GCNII since it is designed for deep models, and all other models are set to be two layers. We run each experiment on three random seeds and make an average. Model training is done on Nvidia Tesla V100 GPU with 32GB memory. In more detail, 
our code depends on cudatoolkit 10.2, pytorch 1.7.0 and torch-geometric 1.7.2.

\subsection{Results}
In this section, we present the results of our experiments and answer four research questions. 

\subsubsection{RQ1. Does \modelname{} outperform SOTA methods on both homophily and heterophily datasets?} 

Table \ref{tab:baseline} shows the comparison of the mean classification accuracy of \modelname{} with other popular GNN models, including three types of recent state-of-the-art methods and six graph neural network based methods. In general, \modelname{} achieves state-of-the-art performance on all of the seven datasets. On homophily datasets, the relative error rate reduction compared to the second place models on Cora, CiteSeer, and PubMed are 7.52\%, 10.57\%, and 8.26\%, respectively. This illustrates that we do better on homophily graphs than other state-of-the-art models. On the other hand, we also have a good performance on heterophily datasets. Compared to the second place models, the relative error rate reduction on Wisconsin, Cornell and Texas are significant and can reach  32.89\%, 11.37\% and 22.20\%, respectively. On Actor, we have the relative error rate reduction 5.25\%. We notice that DNN and GraphSage have a relatively good performance on heterophily graphs compared with other based GNN models. This phenomenon may be caused by their emphasis on capturing the information of the central node. Since nodes in heterophily graphs have relatively weak connections with neighbors, models like GCN, GAT or GIN which focus more on the neighbors may be more easily confused. CPGNN and H2GCN have good performance on many heterophily datasets because they are specially designed for heterophily graphs. However, we still have better performance on heterophily datasets than them and maintain a much better performance on homophily datasets. Besides, our gate mechanism can also be applied to these methods to pursue a higher accuracy since the main idea of the gate mechanism is generic.

\subsubsection{RQ2. How does \modelname{} perform on different splits of datasets?}

We perform experiments on six different splits in which the training set contains 10\% to 60\% nodes on four datasets. Table \ref{tab:homo_split} and \ref{tab:heter_split} show that our model has the best performance on the most of the splits. On the homophily datasets (i.e., Cora and CiteSeer) GCNII achieves similar performance to our methods. It is partially caused by the depth of GCNII being much greater than ours (64 v.s. 2). On the heterophily datasets (i.e., Cornell and Texas), our method significantly outperforms the baselines on different splits. The results show that we can still have a decent gain when the proportion of the training set is smaller, which means that we take good advantage of labels even though they are few.

\begin{table*}[!h]
	\centering
	\caption{Evaluation of our \modelname{} and other popular GNN models on different split of Cora and CiteSeer.}
% 	\resizebox{\textwidth}{!}{
	\begin{tabular}{ccccccccccccc} % {|c|c|c|c|c|c|c||c|c|c|c|c|c|}
	\hline
	\multirow{2}{*}{Model} &
	\multicolumn{6}{c}{Cora} &
	\multicolumn{6}{c}{CiteSeer} \\
% 	\hline
    \cmidrule(r){2-7} \cmidrule(r){8-13} 
	&
	10\% & 20\% & 30\% & 40\% & 50\% & 60\% &
	10\% & 20\% & 30\% & 40\% & 50\% & 60\%
	\\
	\cmidrule(r){1-7} \cmidrule(r){8-13} 
	
	GraphSage &
	80.55  &  84.72  &  84.19  &  86.18  &  87.32  &  87.77&
	67.60  &  \underline{71.88}  &  70.20  &  \underline{73.89}  &  \underline{74.53}   & \underline{76.72}
	\\ 
	GCNII &
	\underline{83.65}  &  \textbf{84.75}  &  \underline{86.24}  &  \underline{87.58}  &  \underline{87.78}  & \underline{87.22} &
	69.39  &  71.42  &  70.29  &  71.81  &  72.63  & 75.11
	\\
	GIN &
	61.61  &  76.30  &  76.49  &  79.81  &  80.84  &  81.93 &
	54.43  &  62.61  &  63.97  &  64.59  &  66.23  & 68.08
	\\
	GAT &
	80.39  &  79.91  &  80.29  &  81.61  &  84.01  &  84.31 &
	68.84  &  68.89  &  67.85  &  70.52  &  72.27  & 71.89
	\\
	GCN &
	82.06  &  83.48  &  85.93  &  87.10  &  87.11  &  \underline{87.22} &
	\underline{69.53}  &  71.65  &  \underline{70.54}  &  71.71  &  72.27  &  74.67
	\\
	\modelname{} & 
	\textbf{83.68} & \underline{84.68} & \textbf{86.27} & \textbf{87.99} & \textbf{88.83} & \textbf{88.69} & 
	\textbf{69.61}  &  \textbf{72.87}  &  \textbf{72.65}  &  \textbf{74.69}  &  \textbf{75.08}  &  \textbf{79.18}
 	\\
	\hline
	\end{tabular}
	  \label{tab:homo_split}
\end{table*}

\begin{table*}[!h]

	\centering
	\caption{Evaluation of our \modelname{} and other popular GNN models on different split of Cornell and Texas.}
% 	\resizebox{\textwidth}{!}{
	\begin{tabular}{ccccccccccccc}
	\hline
	\multirow{2}{*}{Model} &
	\multicolumn{6}{c}{Cornell} &
	\multicolumn{6}{c}{Texas} \\
	\cmidrule(r){2-7} \cmidrule(r){8-13} 
	&
	10\% & 20\% & 30\% & 40\% & 50\% & 60\% &
	10\% & 20\% & 30\% & 40\% & 50\% & 60\%
	\\
	\cmidrule(r){1-7} \cmidrule(r){8-13} 
	
	GraphSage &
	\underline{59.49}  &  \underline{65.78}  &  \underline{72.13}  &  \underline{73.68}  &  \underline{75.00}  & \underline{70.97} &
	\underline{62.03}  &  \underline{69.91}  &  \underline{70.03}  &  \underline{70.89}  &  \underline{70.83}  & \underline{70.97}
	\\ 
	GCNII &
	58.23  &  63.16  &  62.69  &  52.63  &  52.08  & 45.16&
	51.90  &  56.58  &  62.69  &  56.14  &  54.17  & 51.61
	\\
	GIN &
	46.84  &  56.58  &  52.24  &  52.63  &  43.75  & 38.71&
	26.58  &  26.32  &  40.30  &  38.6  &  43.75  & 29.03
	\\
	GAT &
	46.84  &  65.79  &  58.21  &  52.63  &  52.08  & 45.16&
	18.99  &  59.21  &  56.72  &  52.63  &  52.08  & 41.94
	\\
	GCN &
	36.71  &  57.89  &  61.19  &  59.65  &  52.08  & 35.48&
	51.90  &  56.58  &  58.21  &  47.37  &  47.92  & 48.39
	\\
	\modelname{} & 
	\textbf{63.95}  &  \textbf{67.90}  &  \textbf{72.41}  &  \textbf{79.31}  &  \textbf{83.33}  & \textbf{74.27}& 
	\textbf{62.79}  &  \textbf{70.49}  & \textbf{77.64}  &  \textbf{75.51}  &  \textbf{83.33}  & \textbf{81.08}
	\\
	\hline
	\end{tabular}
	  \label{tab:heter_split}
\end{table*}

\subsubsection{RQ3. How does \modelname{} perform on the homophily nodes and heterophily nodes?}

Table \ref{tab:nodes} illustrates the test accuracy of \modelname{} on the nodes range by homophily ratio on different datasets. We rank the nodes in the whole dataset by their homophily ratio and split them into five separations. Homophily of nodes can be seen as the probability that one node has the same label as its neighbor node. As shown in the table, both the nodes have a low homophily rate and high homophily rate have good performance. We suppose our bi-kernel feature transformation contributes to this phenomenon. 
If a node has a high homophily rate, $W_s$ will contribute more to the performance. Otherwise,  $W_d$ will have more contribution.
Since the test dataset may have a distribution shift, there may be no nodes in one proportion of nodes, especially on a small dataset, like Texas which have only 183 and 251 nodes. The test accuracy of \modelname{} may have some drops on some certain proportion, for example, like 0 \textasciitilde 20 homophily ratio proportion on Cora. That is because the lack of test sample lead to a variance in the results. Cora only has 1.6\% nodes in this proportion, so the variance is large.

\begin{table*}[h]
\centering
\caption{Mean classification accuracy of nodes range by homophily ratio in \modelname{}.}
\begin{tabular}{ccccccccccc}
\hline
Hom. Ratio Prop.(\%) & Cora      & CiteSeer & PubMed   & Cornell & Texas & Wisconsin &Actor\\ \hline
0 \textasciitilde 20  & 28.91 &   26.66    &   61.42   &      64.28 &      80.00 &      79.31 & 30.97            \\ 
20 \textasciitilde 40 & 50.00&   42.10   &   71.34    &     100.00 &      N/A &      77.77 &   42.13       \\ 
40 \textasciitilde 60 & 75.00 &  66.66    &    88.98   &      100.00 &      100.00 &      100.00 &  47.14       \\ 
60 \textasciitilde 80 & 91.66  &   96.07    &   92.53   &        100.00 &     N/A &     66.66 &   47.91         \\ 
80 \textasciitilde 100 & 97.86 &   92.44   &  94.46    &      100.00 &     N/A &      50.00 &     42.85    \\ \hline
\end{tabular}
\label{tab:nodes}
\end{table*}

\begin{table*}[h]
\centering
\caption{Mean classification accuracy of the gate in \modelname{}.}
\begin{tabular}{cccccccccc}
\hline
Cora      & CiteSeer & PubMed   & Cornell & Texas & Wisconsin & Actor\\ \hline
% \modelname{}(GCN)    &      &     &     &   &  &  & & & \\ 
82.39 $\pm$ 0.31 & 90.88 $\pm$ 0.22  &  76.07 $\pm$ 0.55  & 65.10 $\pm$ 0.34 & 66.46 $\pm$ 0.25 & 97.67 $\pm$ 0.24 & 78.50 $\pm$ 0.42 \\ \hline
\end{tabular}
\label{tab:gate}
\end{table*}
\subsubsection{RQ4. Does \modelname{} learns the gate well?}

Table \ref{tab:gate} shows the test accuracy of the gate learned by \modelname{}. The ground truth is whether the neighbor node labels are consistent. It turns out that \modelname{} has learned the gate well on homophily graphs and some heterophily graphs.
Although the gate accuracy in Cornell and Texas datasets is worse than other datasets, the final node classification accuracy of these two datasets is much better than baselines. It suggests that the effectiveness of our method does not completely depend on the accuracy of the Gate. As long as the accuracy reaches a certain level, the performance of our model can be improved.

\section{Related Work}

Recently, Graph Neural Networks (GNNs) have been popular for graph-structured data on the node classification problems for it shows the strong ability to mine graph data. 
Deferred proposed one of the early versions of GNNs by generalizing convolutional neural networks (CNNs) for regular grids (e.g., images) to irregular grids (e.g., graphs) \cite{defferrard2016convolutional}. Kipf et.al proposed Graph Convolution Network (GCN) \cite{kipf-gcn} by simplifying the convolution operation on graphs. After that, other GNN models are proposed to improve the representative power on graph tasks in different application scenarios, such as GAT \cite{gat}, GIN \cite{gin} and GraphSAGE \cite{sage}. GAT introduced an attention mechanism to distinguish the importance of each neighbor node. GIN changes the aggregation function to enhance the ability of GNN for graph classification. GraphSAGE was designed for inductive learning for large graphs. 
These methods are generally proposed for the common graphs which satisfy the homophily assumption while other GNNs attempt to break the limits, including MixHop \cite{mixhop}, Geom-GCN \cite{pei2019geom}, and GCNII \cite{gcnii}. 
MixHop and Geom-GCN designed aggregation schemes that go beyond the immediate neighborhood of each node. The residual layer is introduced to leverage representations from intermediate layers to the previous layers. 
These models work well when the homophily ratio of the graph is high.

Recently, methods to research on the graph of heterophily arise as the previous GNNs fail to model these graphs. Xiao Wang et.al. proposed AMGCN \cite{AMGCN} to introduce new graphs generated by the similarity of features and use three separate GCN to learn the node representations simultaneously. CPGNN \cite{CPGNN} incorporates into GNNs a compatibility matrix that captures both heterophily and homophily by modeling the likelihood of connections between nodes in different classes. It overcomes two drawbacks of existing GNNs: it enables GNNs to learn from graphs with either homophily or heterophily appropriately and can achieve satisfactory performance even in the cases of missing and incomplete node features. 
Jiong Zhu et.al. \cite{zhu2020beyond} identified a set of key designs, including ego- and neighbor-embedding separation, higher-order neighborhoods, and a combination of intermediate representations, to boost the learning from the graph structure under heterophily.

\section{Conclusion And Future Work}
In this paper, by analyzing the node-level homophily ratio, we find that the homophily levels may have significant differences in different local regions of a graph. Even for a graph with good overall homophily property, many nodes are still in the mix states, i.e., its neighbors have quite inconsistent labels. Furthermore, by both the empirical experiment results and the theoretical analysis, we demonstrate the fundamental incapability of the traditional GNN models, e.g., GCN, on modeling such mix states where a node has homophily neighbors and heterophily neighbors mixed together. In order to model both homophily and heterophily, we propose the novel gated bi-kernel GNN model, which learns two kernels to model homophily and heterophily, respectively, and learns the selection gate to make the corresponding kernel selection. The significant gain of the extensive evaluation results demonstrates the effectiveness of the approach. 

The key idea of our approach is generic, which uses the selection gate and the multiple kernels to model the different kinds of neighbor relationships. 
In the next step, we plan to use the multi-kernel to model heterophily in a finer-grained manner for multi-class node classification where the total number of node classes is bigger than two. Under such a situation, there are multiple class label combinations, all belonging to heterophily. We could extend the current bi-kernel design to the muti-kernel design to model each kind of heterophily, i.e., each possible pair of class labels, respectively.

\section{Acknowledgements}
We want to thank Haitao Mao at UESTC and Xu Chen at PKU for their constructive comments on this paper.

\bibliographystyle{ACM-Reference-Format}
\balance 
\bibliography{ref}

\end{document}